\pdfoutput=1
\documentclass{amsart}
\usepackage{color}

\usepackage[utf8]{inputenc}

\newcommand{\bce}{\begin{center}}
\newcommand{\ece}{\end{center}}

\newcommand{\kl}[1]{\left(#1\right)}
\newcommand{\ekl}[1]{\left[#1\right]}

\newcommand{\abs}[1]{\left\vert#1\right\vert}

\newcommand{\norm}[1]{\left\| #1 \right\|}

\newcommand{\gm}{\gamma}

\newcommand{\kp}{\kappa}

\newcommand{\lm}{\lambda}

\newcommand{\Hcl}{\mathcal{H}}

\newcommand{\Ncl}{\mathcal{N}}

\theoremstyle{plain}

\theoremstyle{remark}

\theoremstyle{definition}

\newcommand{\trace}{\operatorname{trace}}
\newcommand{\cost}{\operatorname{cost}}
\newcommand{\ga}{g_{\lambda}}
\newcommand{\ra}{r_{\lambda}}
\newcommand{\Sx}{S_{x}}
\newcommand{\Bnu}{B_{\nu}}
\newenvironment{frcseries}{\fontfamily{frc}\selectfont}{}
\newcommand{\textfrc}[1]{{\frcseries#1}}

\newcommand{\littleo}{\textfrc o}
\newcommand{\bigo}{\mathcal O}

\numberwithin{equation}{section}

\theoremstyle{plain}

\newtheorem{theorem}{Theorem}
\newtheorem{proposition}{Proposition}

\newtheorem{lemma}{Lemma}
\newtheorem{cor}{Corollary}

\theoremstyle{definition}
\newtheorem{ass}{Assumption}
\newtheorem*{xmpl}{Example}
\theoremstyle{remark}
\newtheorem{remark}{Remark}

\title[Nystr\"{o}m subsampling for Low smoothness] {Analysis of
  regularized Nystr\"{o}m subsampling for regression functions of low
  smoothness}

\author {
  Shuai Lu
}
\address{School of Mathematical Sciences, Fudan
    University, 200433 Shanghai, China}
\email{slu@fudan.edu.cn}
\author{
  Peter Math\'{e}
}
\address{Weierstrass Institute for Applied Analysis and
    Stochastics, Mohrenstraße 39, 10117 Berlin, Germany}
\email{peter.mathe@wias-berlin.de}
\author{
 Sergiy Pereverzyev Jr.
}
\address{Department of Neuroradiology,
    Medical University of Innsbruck, Anichstraße 35, 6020 Innsbruck,
    Austria}
\thanks{Shuai Lu is supported by NSFC (No.11522108, 91730304), Shanghai Municipal Education Commission (No.16SG01) and Special Funds for Major State Basic Research Projects of China (2015CB856003). Sergiy Pereverzyev Jr. gratefully acknowledges the support of the Austrian Science Fund (FWF): project P 29514-N32.}
\email{sergiy.pereverzyev@i-med.ac.at}

\date{\today}

\begin{document}



\begin{abstract}
  This paper studies a Nystr{\"o}m type subsampling approach to large
  kernel learning methods in the misspecified case, where the
  target function is not assumed to belong to the reproducing kernel
  Hilbert space generated by the underlying kernel. This case is
  less understood, in spite of its practical importance. To model such
  a case, the smoothness of target functions is described in terms of
  general source conditions. It is surprising
    that almost for the whole range of the source conditions,
    describing the misspecified case, the corresponding learning rate
    bounds can be achieved with just one value of the regularization
    parameter. This observation allows a formulation of mild
  conditions under which the plain Nystr{\"o}m subsampling can be
  realized with subquadratic cost maintaining the guaranteed  learning rates.
\end{abstract}

\maketitle


\section{Introduction}
The supervised learning can often be formalized as the problem of
minimizing the expected squared loss
\begin{align*}
  \mathcal{E}(f)=\int\limits_{X\times Y} (f(x)-y)^2d\rho(x,y),
\end{align*}
given a training set $z=\left\{z_i,i=1,2,...,\left|z\right| \right\}$
of samples $z_i=(x_i,y_i)$ drawn independently from a fixed but
unknown (joint)  probability distribution $\rho$ on $Z=X\times Y$,
where $X$ is a set of $d$-dimensional input vectors $x$, and $Y$ is a
set of corresponding outputs labeled by real numbers $y$. Here we
  denote $\left|z\right|$ be the number of the observations. The
primary objective is the regression function $f_\rho$ that minimizes
$\mathcal{E}(f)$ and can be written as
\begin{align*}
  f_\rho(x)=\int\limits_{Y} yd\rho(y|x), \quad  x\in X,
\end{align*}
where $\rho(y|x)$ is the conditional distribution at $x$ induced by
$\rho$ such that $\rho(x,y)=\rho(y|x)\rho_X(x)$, and $\rho_X$ is the
marginal probability measure on $X$.

Since the conditional distribution $\rho(y|x)$
is unknown, the above integral representation for $f_\rho$ is of
no help in practice, and the goal is to find an estimator $\hat f_z$,
on the base of the given training data $z$, that approximates the
unknown regression function~$f_\rho$ well with high probability.

Ideally, a good estimator $\hat f_z$ should have small excess loss
$\mathcal{E}(\hat f_z)-\mathcal{E}(f_\rho)$. Due to a version of
Fubini's theorem we have
\begin{align*}
  \mathcal{E}(\hat
  f_z)-\mathcal{E}(f_\rho)=\left\|\hat f_z-f_\rho\right\|_\rho^2,
\end{align*}
where
$\left\|\cdot\right\|_\rho:=\left\|\cdot\right\|_{L{_2}(X,\rho_X)}$ is
the norm in the space $L{_2}(X,\rho_X)$ of square integrable functions
with respect to the marginal probability measure. Therefore, the
standard way of measuring the performance of the estimator $\hat f_z$
is by studying its convergence to $f_\rho$ in
$ \left\|\cdot\right\|_\rho$-norm.

In kernel machine learning the estimator~$\hat f_z$ is sought within
some  hypothesis space,  often taken to be a
reproducing kernel Hilbert space $\mathcal{H}_K$ associated with a
Mercer kernel $K:X\cdot X\rightarrow\mathbb{R}$. The space~$\mathcal
H_{K}$ is then defined to be the
closure of the linear span of the set of functions $K_x=K(\cdot,x)$,
$x\in X$, with the inner product satisfying
\begin{align*}
  \left\langle K_x,K_{y} \right\rangle_{\mathcal{H}_{K}} := K(x,y),\quad x,y\in X.
\end{align*}

One of the main drawbacks of kernel learning machines is that the
storing and manipulating the kernel Gram matrix
$\mathbb{K}_{\left|z\right|}={\left\{K(x_i,x_j)\right\}}^{\left|z\right|}_{i,j=1}$
require $\bigo(\left|z\right|^2)$ space, and the amount of computations
required to find $\hat f_z\in \mathcal{H}_{K}$ scales as
$\bigo(\left|z\right|^3)$, that can become intractable in the case of the
so-called Big Data, when $\left|z\right|$ grows.
The Nystr{\"o}m type subsampling \cite{Williams,Smola} is a popular
tool for overcoming these limitations. 

Up to now, the theoretical analysis of the Nystr{\"o}m approach has
been carried out extensively in the well-specified case, when
the regression function $f_\rho\in\mathcal{H}_{K}$
\cite{Bach2013,Kriukova2017,LP2013,Myleiko2017,Rudi2015,Rudi2017FALKONAO}. In
the present paper we concentrate ourselves on the misspecified case
such that $f_\rho\in L_2 (X,\rho_X)\setminus\mathcal{H}_{K}$,
which is much less understood, in spite of its practical importance.

The quality of the approximation~$\hat f_{z}$ depends on smoothness
properties of the underlying regression function~$f_{\rho}$, often
given in terms of source conditions and the canonical inclusion
operator $J_K:\mathcal{H}_K\hookrightarrow L_{2}(X,\rho_X)$.

We highlight that the misspecified case yields the fact that
the (unknown) target function $f_{\rho}$ does not belong
to~$\mathcal H_{K}$, and hence that for the regularized empirical
risk functional~$T_{z}^{\lm}$ from~(\ref{eq_generalfunctional}) below, we will
have that $T_{z}^{\lambda}(f_{\rho}) = +\infty$. Such an oversmoothing
penalty term is not standard in classic regularization theory, see
e.g.~\cite{EHN1996}, but it has gained attention in numerical
differentiation \cite{WJC2002,WWY2006} and regularization in Hilbert
scales \cite{Natterer1984,HM2018}. 
In the present setting the convergence analysis shall be carried out in the
norm in the space $L_2(X, \rho_X)$ instead of the norm in the
reproducing kernel Hilbert space $\mathcal{H}_K$.

Due to \cite{Rudi2015} the Nystr{\"o}m approach can be seen as a
combination of random projections with a regularization scheme, and
the regularization theory tells us that such a scheme should have
enough qualification to utilize the whole smoothness of $f_\rho$. On
the other hand, from this perspective it follows that, because of low
smoothness of $f_\rho$, even a scheme with a modest qualification,
such as the standard Tikhonov regularization, is sufficient for the
misspecified case. For this reason, in the present study we restrict
ourselves to the Nystr{\"o}m subsampling for Tikhonov regularization
known also as the kernel ridge regression (KRR).

The learning rate (i.e., the convergence rate of the approximant
  to the target function $f_\rho $ in
  $\left\|\cdot\right\|_\rho $-norm) of KRR in the misspecified case
was first studied in \cite{Smale2007}. As one may see from that study,
for $f_\rho\in L_2 (X,\rho_X) \setminus\mathcal{H}_{K}$ the learning
rate of KRR cannot be in general described by the same formula as in
the well-specified case $f_\rho\in\mathcal{H}_{K}$. A uniform
description for both cases was obtained in \cite{Steinwart} under
additional assumptions on the inclusion operator $J_{K}$,
which may not be always satisfied. To the best of our knowledge, the
best known learning rates, that are valid for KRR with arbitrary
Mercer kernel functions $K$, have been recently given in
\cite{ShaoboLin}. In the present research we study conditions under
which the above mentioned rates can be achieved at a subquadratic
cost (with respect to
the number of observations $\left|z\right|$)
by KRR combined with the Nystr{\"o}m approach.

The paper is organized as follows. In the next section we recall KRR
setting. Then we follow \cite{Rudi2015} and consider the Nystr{\"o}m
approach to KRR as a projection method regularized by Tikhonov
regularization. In Section 3 we estimate the learning rate of KRR
combined with the plain Nystr{\"o}m subsampling in the misspecified case.
The technical proofs are given separately in Section~\ref{sec:proofs}.

In contrast to previous studies, we employ general source
conditions to measure the smoothness
$f_\rho\in L_2 (X,\rho_X)\setminus \mathcal{H}_{K}$. One important
and interesting observation here is that almost for the whole range of
source conditions describing the misspecified case the corresponding
learning rate bounds can be achieved with the same value of Tikhonov
regularization parameter that can be chosen a priori and without any
knowledge of the smoothness of $f_\rho$. This observation allows us to
formulate simple conditions under which the plain Nystr{\"o}m
subsampling can be realized with subquadratic cost, still maintaining
guaranteed learning rates.

\section{KRR with Nystr\"{o}m subsampling}

Recall that in KRR the goal is to approximate $f_\rho$ by the
minimizer $f_{z}^\lambda$ of the regularized empirical risk functional

\begin{align}\label{eq_generalfunctional}
  T_{z}^\lambda(f)&
                    :=\left|z\right|^{-1}\sum\limits_{i=1}^{\left|z\right|}(f(x_i)-y_i)^2+\lambda\left\|f\right\|^2_{\mathcal{H}_{K}},\quad
                    f\in \mathcal H_{K}.
\end{align}
For the subsequent analysis we shall use the identical
operator~$I\colon \mathcal H_{K}\to \mathcal H_{K}$,
the canonical inclusion~$J_{K}\colon \mathcal H_{K}\to  L_2
(X,\rho_X)$,
and the sampling
operators~$\Sx:\mathcal{H}_K\rightarrow\mathbb{R}^{\left|z\right|}$,
given as~$\Sx f=(f(x_i))^{\left| z\right|}_{i=1} $,
and its adjoint~$\Sx^{\ast}:\mathbb{R}^{\left|z\right|}\rightarrow\mathcal{H}_K$.
It is known, cf.~\cite{ShaoboLin}, that the product of the inclusion
operator $J_K$ and its adjoint is the integral operator defined by
\begin{align}\label{eq:jkjkast}
  \left( J_{K}J_{K}^* \right) f(x)= \int\limits_{X} K(x,t) f(t)
  d\rho_X(t),\quad f\in L_2 (X,\rho_X),\ x\in X.
\end{align}		

The celebrated representer theorem of G. S. Kimeldorf and G. Wahba
tells us that the minimizer of~(\ref{eq_generalfunctional}) has the form
\begin{align*}
  f_z^\lambda &=\sum\limits_{i=1}^{\left|z\right|}c_i\cdot K(\cdot,
  x_i), \
  c=(c_i)^{\left|z\right|}_{i=1}=(\mathbb{K}_{\left|z\right|}+\lambda
  \mathbb{I})^{-1}\mathbb{Y},
\end{align*}
where $\mathbb{I}$ is the $\left|z\right|\times\left|z\right|$
diagonal identity matrix and
$\mathbb{Y}=(y_i)^{\left|z\right|}_{i=1}$.

It is clear that KRR has at least quadratic computational cost
$\bigo(\left|z\right|^2)$, as this is the cost of computing the
kernel Gram matrix
$\mathbb{K}_{\left|z\right|}=\lbrace {K(x_{i},
  x_{j})}\rbrace_{i,j=1}^{\left|z\right|}$.
Therefore, in the setting, where $\left|z\right|$ is large,
one tries to avoid the computation of the minimizer
$f_{z}^\lambda$.

\subsection{Plain Nystr\"{o}m subsampling}
\label{sec:subsampling}

In the Nystr{\"o}m algorithms this is done by replacing
$\mathbb{K}_{\left|z\right|}$ with a smaller low-rank matrix obtained
by random subsampling of columns of $\mathbb{K}_{\left|z\right|}$.
In the forthcoming
analysis, we restrict attention to the so-called \emph{plain
  Nystr{\"o}m subsampling}
approach, where the points $(x_i, y_i)$ forming the subsample~$z^\nu$ are sampled
uniformly at random without replacement from the training set
$z$.

An important observation made in \cite{Rudi2015} is that the Nystr{\"o}m
subsampling can be interpreted as a restriction of the minimization of
$T_z^\lambda(f)$ to the (randomly chosen) space
\begin{align*}
  \mathcal{H}_K^{z^\nu}:=\lbrace{f : f=\sum\limits_{x_i:(x_i,y_i)\in
  z^\nu}}c_iK(\cdot,x_i), \ c_i \in \mathbb{R}\rbrace\subset\mathcal{H}_K ,
\end{align*}
where $z^\nu$ is a randomly selected subset of $z$ with the
cardinality $\left|z^\nu\right|\ll\left|z\right|$.

We let~$P_{z^\nu}:\mathcal{H}_K\rightarrow \mathcal{H}^{z^\nu}_K$ be
the orthogonal projection operator in $\mathcal{H}_K$ with the range
$\mathcal{H}^{z^\nu}_K$.
It is clear that $P_{z^\nu}$ has a probabilistic character and depends
on the way we perform the subsampling $z^\nu$.

In the analysis the composition~$\Bnu:= \Sx
P_{z^\nu}\colon \mathcal H_{K}\to \mathbb R^{|z|}$ will be relevant.
 Then the minimizer
$f_{z,z^\nu}^\lambda$ of $T_z^\lambda (f)$ over
$\mathcal{H}_K^{z^\nu}$ is given as
\begin{align} \label{eq:2.3} f_{z,z^\nu}^\lambda &=(\lambda
  I+P_{z^{\nu}}\Sx^{\ast}\Sx P_{z^\nu})^{-1} P_{z^\nu} \Sx^{\ast} \mathbb{Y},\\
&= P_{z^{\nu}}(\lambda
  I+ \Bnu^{\ast}\Bnu)^{-1} \Bnu^{\ast} \mathbb{Y},\notag
\end{align}
where the latter follows because~$f_{z,z^\nu}^\lambda\in \mathcal{H}_K^{z^\nu}$.
Note that $f^\lambda_{z,z^\nu}$ can be computed with a computational
cost
$$
\cost(f^\lambda_{z,z^\nu}) = \mathcal
O(\left|z\right|\cdot\left|z^{\nu}\right|^2),\quad \text{as}
\left|z\right|\geq \left|z^{\nu}\right|\to\infty.
$$
(see, e.g., \cite{Rudi2015}).


\subsection{Assumptions}
\label{sec:assumptions}

The answer to the learning rate depends on additional assumptions.
The first two assumptions, concerning properties of the underlying
kernel and the noise moments, will not be referenced explicitly throughout the study.
\begin{ass}
  [kernel properties]\label{ass:kernel}
The kernel~$K\colon X\times X\to \mathbb R$ is
continuous, symmetric, positive definite and
$$
  \sup\limits_{ x\in X } K(x,x) = \kp < \infty.
$$
Then it is clear that
$$
  \sup\limits_{x\in X} \norm{ K_x }_{  L_2 \kl{  X, \rho_X  }  }  \leq
  \sup\limits_{x\in X} \norm{ K_x }_{  C(X)  } \leq
  \sup\limits_{x\in X} \norm{ K_x }_{  \Hcl_K  }^{2} = \kp.
$$
\end{ass}

Under Assumption~\ref{ass:kernel} the operator~$J_{K}J_{K}^*$
from~(\ref{eq:jkjkast}) has a
finite trace. Specifically, for each~$\lambda>0$ it holds that
$$
\mathcal N_x(\lm) :=\langle K(\cdot,x),(\lambda
I+J^*_KJ_K)^{-1}K(\cdot,x)\rangle_{\mathcal{H}_K} = \| (\lm +
J^*_KJ_K)^{-1/2} K_{x}\|^{2}< \infty.
$$
We highlight the related quantities
\begin{align}
	\label{eq:2.2n}
	\Ncl_{\infty}(\lm) &:=\sup \limits_{x\in X} \Ncl_x(\lambda) \quad (\leq \kp/\lm),
  \\
\intertext{and }
\mathcal N(\lm) &:=\int\limits_X
  \mathcal N_x(\lambda)d\rho_X(x)=\trace\lbrace(\lambda I+J^*_KJ_K)^{-1}J^*_KJ_K\rbrace.
\end{align}
The function~$\mathcal N$ measures the capacity of the
RKHS $\mathcal H_{K}$ in the space~$L_{2}(X,\rho_{X})$, and it is
called {\it the effective dimension}. It is well known that this is a
decreasing function of~$\lm$ with~$\lim_{\lm\to 0+}\mathcal
N(\lm)=\infty$, provided that the RKHS~$\mathcal H_{K}$ is infinite
dimensional, c.f. \cite{Zha05}.
Extended discussion on properties of the effective dimension for general operators can be found in \cite{LM2014}.

For the efficiency of the Nystr\"{o}m
subsampling we shall need an additional assumption on the kernel,
specified in Assumption~\ref{ass:kernel-src}, below.

\begin{ass}
  [noise moments]\label{ass:noise}
The family of random variables~$\varepsilon_{x}:= y - f_{\rho}(x),\
x\in X$ has all moments~$p\geq 2$, which satisfy
$$
\mathbf E| \varepsilon_{x}|^{p} \leq \frac 1 2 p! M^{p-2}
\sigma^{2},\quad x\in X, \ \text{ a.e.},
$$
for some positive constants~$M$ and~$\sigma$.
\end{ass}

Next, an assumption is made on the underlying smoothness of the regression function~$f_\rho$.
\begin{ass}
  [source condition]\label{ass:smoothness}
There is an operator concave index function\footnote{
A function~$\varphi:\ekl{0, d}   \rightarrow[0,\infty)$
is called an index function if it is continuous, strictly increasing,
obeying $\varphi (0)=0$. It is called operator concave if for
self-adjoint operator~$C,C_{1}\colon L_2 (X,\rho_X) \to L_2 (X,\rho_X) $ we
have that~$\varphi\left( \frac 1 2 (C + C_{1}) \right) \geq \frac 1 2 \left(
  \varphi(C) + \varphi(C_{1})\right)$.
}~$\varphi\colon [0,d]\to[0,\infty) $, for some~$d> \|J_{K}J_{K}^*\|$, such that
$$
f=\varphi (J_{K}J_{K}^*)v_f, \  \left\|v_f\right\|_\rho \leq 1.
$$
The function~$t \mapsto \sqrt t/\varphi(t)$ is nondecreasing.
\end{ass}

\begin{remark}\label{rem:smoothness}
  First, from~\cite{Mathe2008} we know that for every
$f\in L_2 (X,\rho_X)$ and $\varepsilon>0$ there exists an
index function
$\varphi:\ekl{ 0, \left\|J_{K}  J_{K}^* \right\| }
\rightarrow[0,\infty)
$
such that
\begin{align} \label{eq:2.1} f=\varphi (J_{K}J_{K}^*)v_f, \
  \left\|v_f\right\|_\rho \leq (1+\varepsilon)\left\| f \right\|_\rho.
\end{align}
In the
context of learning it is used starting from the paper
\cite{Smale2007}, where $f=f_\rho$ was assumed to satisfy
\eqref{eq:2.1} with $\varphi(t)=t^r, r\in (0,1].$

Secondly, the following is known. If the function $t\rightarrow\sqrt t/\varphi(t)$ is
nonincreasing, then the image of $\varphi (J_{K}J_{K}^*)$ is contained
in $\mathcal{H}_K$. Therefore, in order to treat the low-smoothness
case we assume that
$t\rightarrow\sqrt{t}/\varphi(t)$ is nondecreasing.
Thus, the misspecified case studied here corresponds
to \eqref{eq:2.1} with $\varphi(t)$ increasing not faster than
$\sqrt{t}$.

Moreover, as in \cite{Kriukova2017}, in order to control the effect
of subsampling, we assume that $\varphi$ is operator concave on
$[0,d],\ d > \left\|J_{K}J_{K}^*\right\|$. Note that previously
considered H{\"o}lder-type index functions
$\varphi(t)=t^r, r\in (0,\frac{1}{2}]$, as well as logarithmic
functions $\varphi(t)=\log^{-r}\frac{1}{t},\ { 0 < r\leq 1}$, are
operator monotone, and hence operator concave. An important
implication of operator monotonicity is that there is a number
$d_\varphi$ depending only on $\varphi$ such that for any self-adjoint operators
$C,C_1:L_2(X,\rho_X)\rightarrow L_2(X,\rho_X)$ with spectra in $[0,d]$
it holds
\begin{align} \label{eq:2.2}
  \left\|\varphi(C)-\varphi(C_1)\right\|_{L_2(X,\rho_X)\to
    L_2(X,\rho_X)}\leq d_\varphi
  \varphi\left(\left\|C-C_1\right\|_{L_2(X,\rho_X)\to
    L_2(X,\rho_X)}\right).
\end{align}
\end{remark}


\section{Main results}
\label{sec:results}


\subsection{Error bound for Nystr{\"o}m subsampling in the misspecified case}

We formulate the main result below. For this to hold we must obey  the following
relations for the parameter~$\lm>0$, the overall sample
size~$\left|z\right|$, and the subsample size~$\left|z^{\nu}\right|$.
Given, for~$0< \delta < 1$, a confidence level~$1-\delta$, we require that
\begin{align}\label{eq:2.5}
  \left|z^{\nu}\right| &\geq c\mathcal N_\infty
                         (\lambda)\log\dfrac{1}{\lambda}\log\dfrac{1}{\delta},\\
\intertext{and}
  \lambda & \in \left[c{\left|z\right|}^{-1}\log {\dfrac{{\left|z\right|}}{\delta}},\ \left\|J^*_KJ_K\right\|_{\mathcal{H}_K\rightarrow\mathcal{H}_K} \right].\label{eq:3.3}
\end{align}
Concerning the choice of the size~$|z^{\nu}|$ of the subsample, two competing
goals are relevant. First, it should be large enough to maintain the
learning rate as this was obtained by using the full sample~$z$. On
the other hand, it should be as small as possible to reduce the
computational burden. Here this choice in (\ref{eq:2.5}) is analyzed in the low smoothness situation.

Here and in the sequel, we adopt the convention that $c$ denotes a generic positive coefficient, which can vary from estimation
to estimation and may only depend on basic parameters, such as $K$, $\rho$.
Also, for functions~$a$, $b$ depending on $\lm$ or $\abs{z}$, respectively, the relation $a \asymp b$ means that
$a = \bigo(b)$ and $b = \bigo(a)$ as $\lm\to 0$, or $\abs{z}\to\infty$.

Note that the Nystr{\"o}m approximant from~\eqref{eq:2.3}
represents an element of $\mathcal{H}_K$, and to estimate
$\mathcal{E}(f^\lambda_{z,z^\nu})$ we need to embed
$f^\lambda_{z,z^\nu}$ in $L_2(X,\rho_X)$. Then the error decomposes as
  \begin{align}\label{eq:3.6}
    \begin{split}
\left\|J_Kf^\lambda_{z,z^\nu}-f_{\rho}\right\|_{\rho}
&\leq \left\|f_{\rho}-J_K(\lambda I+
  \Bnu^{\ast}\Bnu)^{-1}P_{z^\nu}J^*_Kf_{\rho}\right\|_{\rho}
    \\
&+ \left\|J_K(\lambda I+\Bnu^{\ast}\Bnu)^{-1}
  P_{z^\nu}(J^*_Kf_{\rho}-\Sx^{\ast}\mathbb{Y})\right\|_{\rho},
    \end{split}
  \end{align}
which can be regarded as decomposition into approximation error and the
sample error, respectively. By estimating both terms in the right-hand side of (\ref{eq:3.6}), we establish the main error estimate.
\begin{theorem}\label{th:1}
  Assume that in the plain Nystr{\"o}m subsampling the values
  $\left|z^{\nu}\right|$ and $\lambda$ satisfy \eqref{eq:2.5} and
  \eqref{eq:3.3}. If $f_{\rho}$ obeys Assumption~\ref{ass:smoothness}
  for the index function~$\varphi$, then with probability at
  least $1-\delta$ we have
  \begin{align*}
& \left\|f_{\rho}-J_K(\lambda I+ \Bnu^{\ast}\Bnu)^{-1}P_{z^\nu}J^*_Kf_{\rho}\right\|_{\rho} \leq c\varphi(\lambda)\log {\dfrac{1}{\delta}}
\left( 1+\sqrt{\dfrac{{\Ncl(\lambda)}}{\lambda{\left|z\right|}}}\right); \\
& \left\|J_K(\lambda I+\Bnu^{\ast}\Bnu)^{-1}
  P_{z^\nu}(J^*_Kf_{\rho}-\Sx^{\ast}\mathbb{Y})\right\|_{\rho} \leq c\sqrt \lm
  \log \dfrac{1}{\delta}\left( 1 + \sqrt{\dfrac{\Ncl(\lambda)}{\lambda\left|z\right|}}\right);
  \end{align*}
  and the total error estimate
  \begin{align*}
\left\|J_Kf^\lambda_{z,z^\nu}-f_{\rho}\right\|_{\rho}\leq c\varphi(\lambda)\log {\dfrac{1}{\delta}}
\left( 1+\sqrt{\dfrac{{\Ncl(\lambda)}}{\lambda{\left|z\right|}}}\right) .
  \end{align*}
\end{theorem}
It is interesting to observe that the approximation error
dominates the sample error, which is not standard in regularization
theory. This is a consequence of the misspecified source condition,
Assumption~\ref{ass:smoothness}, for a function~$\varphi$
with~$\varphi(\lm)\geq c \sqrt\lm$. We will provide more explanation in Remark \ref{rem:approxi} after the proof of the above theorem, which has been postponed in Section \ref{se_4}.

\subsection{Parameter choice}
\label{sec:parameter}

A somehow surprising message of the above theorem is that the $\lambda$-dependent term
\begin{align*}
  \theta_{\varphi}(\lambda)=\varphi(\lambda)\left(1+\sqrt{\frac{\Ncl(\lambda)}{\left|z\right|\lambda}}\right),\quad \lm>0.
\end{align*}
bounding (the square root of) the excess loss
$\mathcal{E}(f^\lambda_{z,z^\nu})-\mathcal{E}(f_{\rho})$ attains its
minimum (up to a constant factor) at a value of the regularization
parameter $\lambda=\lambda_0$, which can be chosen a priori and does
not require the knowledge of the index function $\varphi$. Precisely,
let~$\lambda_{0} = \lm_{0}(\left|z\right|)$ solve the equation
\begin{equation}
  \label{eq:lambda0}
  \mathcal N(\lambda)=\lambda\left|z\right|.
\end{equation}
Notice that this equation always has a unique solution, and that it
does not depend on the underlying smoothness, as expressed in the
function~$ \varphi$.
Also, as~$\left|z\right|\to \infty$ we have
that~$\lm_{0}(\left|z\right|)\to 0$.


\begin{cor}\label{cor:lambda0}
For any index function~$\varphi$ in Assumption~\ref{ass:smoothness} we have
  \begin{align}
    \varphi(\lambda_0)\leq \min\limits_{\lambda}\theta_{\varphi}(\lambda)\leq2 \varphi(\lambda_0),\label{eq:lambda-minr}
  \end{align}			
  where $\lambda_0$ is chosen in~(\ref{eq:lambda0}).

Consequently, under the conditions of Theorem~\ref{th:1}, and if~$\lm_{0}$
obeys~(\ref{eq:3.3}) then we have that
$$
\|J_Kf^\lambda_{z,z^\nu}-f_{\rho} \|_{\rho} =
\bigo(\varphi(\lm_{0}(|z|))),\quad \text{as}\ \ |z|\to \infty.
$$
\end{cor}

\begin{remark}
  \label{rem:emp-eff-dim}
The effective dimension $\Ncl(\lambda)$ can be rather
accurately estimated from the data (see,
e.g.,~\cite[Prop.~1]{Rudi2015}) that makes the parameter choice
$\lambda=\lambda_0$ practically feasible.
\end{remark}
\begin{remark}
  \label{rem:lm0-feasible}
We comment when the above choice of~$\lm_{0}$  obeys the
condition from~(\ref{eq:3.3}). We claim that this holds true whenever
the effective dimension grows at least as~$\log(1/\lm)$, as~$\lm\to 0$.
Indeed, we have that~$|z|\lm_{0}\leq 1$, and hence that~$\log(|z|)
\leq \log(1/\lm_{0})$, such that in this case we find that
$$
|z|\lm_{0} = \mathcal N(\lm_{0}) \geq \log(1/\lm_{0})\geq \frac 1 2 \log\left(\frac{|z|}{\delta}\right),
$$
provided that for given confidence level~$1-\delta$, the sample
zise~$|z|$ is large
enough.
This condition on the effective dimension is fulfilled for all types of the behavior of the effective
dimension discussed in the literature, (see, e.g., the discussion with
power type behavior in Section~\ref{sec:efficiency}, below).
\end{remark}

\subsection{Full data}
\label{sec:full-data}

Note that in the case when $\left|z^{\nu}\right| = \left|z\right|$, the inequality~\eqref{eq:2.5} is satisfied
because in view of~\eqref{eq:2.2n} and~\eqref{eq:3.3}, $\Ncl_{\infty}(\lm) \log(1/\lm) $ is of
lower order than $\left|z^{\nu}\right| = \left|z\right|$, i.e.
$   \Ncl_{\infty}(\lm) \log(1/\lm)  = \bigo\kl{
\abs{z} \cdot \log^{-1}\abs{z} \cdot \log\log \abs{z}
}  $. Therefore, Theorem~\ref{th:1} has the following corollary.

\begin{cor}
If $z^{\nu} = z$, then under the conditions of Theorem~\ref{th:1}, we have
\begin{align}\label{eq:3.14}
\left\| J_K f^{\lm_0}_{z} - f_{\rho} \right\|_{\rho}=
  \mathcal O\left(\varphi(\lambda_0)\log\dfrac{1}{\delta}\right)\quad
  \text{as}\ |z| \to \infty.
\end{align}
\end{cor}

The error of Nystr\"{o}m subsampling that follows from Theorem~\ref{th:1} and
Corollary~\ref{cor:lambda0}, coincides with the best learning rate known in the
misspecified case for KRR with general Mercer kernels and full data,
i.e $z^{\nu}=z$.

\begin{xmpl}
  We discuss the previously considered H{\"o}lder-type index
functions $\varphi(t)=t^r, \ r\in (0,1/2]$, and under the usual
assumption on the effective dimension
$\Ncl(\lambda)=\bigo(\lambda^{-s}), \ s\in(0,1]$. Then the bound~\eqref{eq:3.14}
is of order $\bigo\left(\left|z\right|^{    -r/(s+1)     }     \right)$.
For KRR with full data, this result is in accordance with~\cite{ShaoboLin}.
\end{xmpl}


\subsection{Efficiency of Subsampling}
\label{sec:efficiency}

Now we are in position to discuss conditions under which the plain
Nystr{\"o}m subsampling achieves \eqref{eq:3.14} with subquadratic
cost~$ \littleo(\left|z\right|^2)$.
In order to actually establish the superiority of the subsampling an
additional assumption is made, borrowed from~\cite{VitRosToi14}.

\begin{ass}
  [source condition for kernel]\label{ass:kernel-src}
There exist $\gamma\in (0,1]$ and
$c_{\gamma}>0$ such that for all $x\in X$ the kernel sections
$K(\cdot,x)\in\mathcal{H}_K$ satisfy the source condition
\begin{align} \label{eq:3.16} K(\cdot,x)= (J^*_KJ_K)^{\gamma
    /2}\upsilon_{x}, \ \left\|\upsilon_x\right\|_{\mathcal{H}_K}\leq
  c_\gamma.
\end{align}
\end{ass}

\begin{remark}\label{rem:kernel}
  As discussed in Remark~\ref{rem:smoothness} there is always an index
  function, say~$\psi$, which guarantees that
$$
K(\cdot,x)= \psi(J^*_KJ_K)\upsilon_{x}, \
\left\|\upsilon_x\right\|_{\mathcal{H}_K}\leq
 (1 + \varepsilon) \| K_{x}\|_{  \Hcl_K  }
\leq (1 + \varepsilon)  \sqrt{\kp} .
$$
Thus, Assumption~\ref{ass:kernel-src} requires this index function to be
of at least power type.
\end{remark}

We mention the following consequence of Assumption~\ref{ass:kernel-src}.
\begin{lemma}\label{lem:gamma-bound}
  Under Assumption~\ref{ass:kernel-src} we have that
$$
 \mathcal N_{\infty}(\lambda) \leq c^2_\gamma \lambda^{\gamma-1}.
$$
\end{lemma}
\begin{proof}
  This simply follows from
\begin{align*}
  \mathcal N_{\infty}(\lambda)  &=\sup\limits_{x\in X}\left\|(\lambda
                         I+J^*_KJ_K)^{-1/2}K(\cdot,x)\right\|^2_{\mathcal{H}_K}\\
&\leq\sup\limits_{x\in X}\left\|\upsilon_x\right\|^2_{\mathcal{H}_K}\sup\limits_{t>0}[(\lambda+t)^{-1/2}t^{\gamma/2}]^2
  \\
                     &\leq c^2_\gamma \sup_{t>0} (\lambda+t)^{-1}t^{\gamma}\leq c^2_\gamma\lambda^{\gamma-1},
\end{align*}
which completes the proof.
\end{proof}

Recall that $f^{\lambda}_{z,z^{\nu}}$ can be computed with a
computational cost
$\bigo(\left|z\right|\cdot\left|z^{\nu}\right|^2)$, and note that
\eqref{eq:2.5} is the only condition on the subsampling size
$\left|z^{\nu}\right|$ that is needed in Theorem~\ref{th:1}. Then from
Corollary~\ref{cor:lambda0} it follows that the Nystr{\"o}m approximation
$f^{\lambda_0}_{z,z^{\nu}}$ realizing the order \eqref{eq:3.14} can be
computed with a computational cost
\begin{align}\label{eq:3.15}
  \cost\left(f^{\lambda_0}_{z,z^{\nu}}\right)=\bigo \left(\left|z\right|\cdot\left(\mathcal
  N_{\infty}(\lambda_0)\log\dfrac{1}{\lambda_0}\right)^2\right).
\end{align}

If we stay with the standard assumption that
$\mathcal N(\lambda)   \asymp   \lambda^{-s}$,
then from the very definition and Lemma~\ref{lem:gamma-bound}, it follows that
\begin{align*}
  \mathcal N(\lambda)\leq \mathcal N_{\infty}(\lambda)\Rightarrow s+\gamma\leq 1.
\end{align*}

In the considered scenario, from~\eqref{eq:lambda0}
and~\eqref{eq:3.15}, we have $\lambda_0  \asymp  \left|z\right|^{-1/(s+1)}$, and the cost can be
bounded as
\begin{align*}
  \cost(f^{\lambda_0}_{z,z^{\nu}})=\bigo\left(\left|z\right|\cdot\left|z\right|^{ 2(1-\gamma)/(s+1) }
	\log^2\left|z\right|\right)=\bigo\left(\left|z\right|^{ (3+s-2\gamma)/(1+s)   }
	\log^2\left|z\right|\right),
\end{align*}
which is subquadratic whenever~$2\gamma+s>1$. We summarize this as
\begin{proposition}\label{prop:2}

Assume that Assumption~\ref{ass:kernel-src} holds true and $\Ncl(\lm) \asymp \lm^{-s}$, $s\in ( 0, 1-\gm ]$.
If $2\gm + s > 1$, then the plain Nystr\"{o}m approximation
$   f_{   z, z^{\nu}  }^{  \lm_0  }   $
can be computed at a subquadratic computational cost, and it preserves the learning rate~\eqref{eq:3.14}
guaranteed for the full amount of data.

\end{proposition}

In particular,
if Assumption~\ref{ass:kernel-src} is satisfied with $\gm > 1/2$, 
then the plain Nystr\"{o}m approximation can always be computed at a subquadratic cost still maintaining guaranteed
learning rates.


\section{Proofs}\label{se_4}
\label{sec:proofs}

\subsection{A regularization perspective to KRR}
\label{sec:regularization}

Here we briefly emphasize the aspects of regularization theory which
will be relevant in the subsequent proofs.
We recall the structure of the estimator~$f_{z,z^{\nu}}^{\lm}$
from~(\ref{eq:2.3}) as
$$
 f_{z,z^\nu}^\lambda = \left(\lambda
  I+\Bnu^{\ast}\Bnu\right)^{-1} \Bnu^{\ast}\mathbb{Y},
$$
with~$\Bnu= \Sx P_{z^{\nu}}$. We can write this as
$ f_{z,z^\nu}^\lambda = \ga(\Bnu^{\ast}\Bnu)\Bnu^{\ast}\mathbb{Y}$,
where we introduced the KRR filter function~$\ga(t):= 1/(t + \lm ),\
t,\lm>0$, applied to the non-negative self-adjoint
operator~$\Bnu^{\ast}\Bnu$ via spectral calculus.
  We shall also employ the fact that for any linear bounded operator,
  say~$B$ acting between Hilbert spaces,  and for any bounded function $g$
  it holds~$g(B^*B)B^*=B^*g(BB^*)$. The corresponding residual
function is given as~$\ra(t) := 1 - \ga(t) t =
\lm/(t+\lm),\ t,\lm >0 $. In particular we have that~$0 < \ra(t) \leq 1$. The impact of the residual function on the
given solution smoothness is measured by its qualification, and we
mention the well known result that
\begin{equation}
  \label{eq:quali}
  \sup_{t>0}\left|\ra(t) \varphi(t)\right| \leq \varphi(\lm),\quad \lm>0,
\end{equation}
provided that the index function~$\varphi$ is such that~$\varphi(t)/t$
is non-decreasing, as this is the case for the functions~$\varphi$
which obey Assumption~\ref{ass:smoothness}.
Applying this for the index function~$t \mapsto t^{q}\varphi(t)$,
with~$0 \leq q \leq 1/2$, we find that this still obeys the assumption
for~(\ref{eq:quali}), and hence, (see, e,g., (16) in \cite{Kriukova2017}), we have that	
  \begin{align}\label{eq:3.9}
    \sup_{t>0}\left|\ra(t)t^q\varphi(t)\right|\leq
    c\lambda^q\varphi(\lambda), \ \text{when }q\in[0,\dfrac{1}{2}].
  \end{align} 	
Finally, by the specific structure we see that~$\ga(t) = \ra(t)/ \lm $,
such that~(\ref{eq:3.9}) yields with~$q:= 1/2$ that
\begin{equation}
  \label{eq:ga-bound}
  \left|\ga(t)\sqrt t \varphi(t) \right| \leq \varphi(\lm)
  \lm^{-1/2},\quad \lm>0.
\end{equation}


\subsection{Probabilistic bounds}
\label{sec:prob-bounds}

We shall also use probabilistic bounds.
  From~\cite[Lem.~6]{Rudi2015} and \cite[Cor.~1]{Myleiko2017} it follows that if
$z^\nu$ is subsampled according to the plain Nystr{\"o}m approach,
then with probability at least $1-\delta$ we have
\begin{align} \label{eq:2.4} \left\|\
    J_K(I-P_{z^\nu})\right\|^2_{\mathcal{H}_K\rightarrow
    L_2(X,\rho_X)} \leq 3\lambda,
\end{align}
provided that~\eqref{eq:2.5} holds.

We also recourse to the following inequality
from~\cite[Lem.~5]{Rudi2015}, which assert
\begin{align} \label{eq:3.2} \left\|(\lambda I+J^*_KJ_K)^{1/2}(\lambda
    I+\Sx^{\ast}\Sx)^{1/2}\right\|_{\mathcal{H}_K\rightarrow\mathcal{H}_K}\leq2,
\end{align}
the latter one is satisfied with probability at least $1-\delta$ if~\eqref{eq:3.3} holds.

Moreover, from Lemma 5.1 \cite{Blanchard2016} it follows that for
$\lambda$ satisfying \eqref{eq:3.3} with probability at least
$1-\delta$ we have
\begin{align} \label{eq:3.4} \left\|(\lambda
    I+J^*_KJ_K)^{-1/2}(J^*_KJ_K-\Sx^{\ast}\Sx)\right\|_{\mathcal{H}_K\rightarrow\mathcal{H}_K}\leq
  c
  \log{\dfrac{1}{\delta}}\sqrt{\dfrac{{\Ncl(\lambda)}}{{\left|z\right|}}},
\end{align}
\begin{align}\label{eq:3.5}
  \left\|(\lambda I+J^*_KJ_K)^{-1/2}(J^*_Kf_{\rho}-\Sx^{\ast}\mathbb{Y})\right\|_{\mathcal{H}_K} \leq
	c \log{\dfrac{1}{\delta}}\sqrt{\dfrac{{\Ncl(\lambda)}}{{\left|z\right|}}}.
\end{align}


\subsection{Proof of Theorem~\ref{th:1}}
\label{sec:proof-thm}

We first mention the following  well known bound, using spectral calculus.
  \begin{align}\label{eq:J_K-bound}
    \begin{split}
    \left\|J_K(\lambda I+J_K^* J_K)^{-1/2}\right\|_{\mathcal{H}_{K}\rightarrow L_2(X,\rho_X)}
 & = \left\|(J^*_KJ_K)^{1/2}(\lambda
   I+J^*_KJ_K)^{-1/2}\right\|_{\mathcal{H}_K\rightarrow \mathcal{H}_K}   \\
 &\leq \sup\limits_{t>0}(t/(\lambda +t))^{1/2}\leq 1.
    \end{split}
  \end{align}
Furthermore, the following result will be used, and we refer to~\cite[Lem.~2\&8]{Rudi2015}. For every
choice~$z^{\nu}$ from the sample~$z$ we have that
\begin{align} \label{eq:3.1} \left\|(\lambda I+ \Sx^{\ast}\Sx)^{1/2}P_{z^\nu} (\lambda
    I+P_{z^\nu}\Sx^{\ast}\Sx P_{z^\nu})^{-1}P_{z^\nu}(\lambda
    I+\Sx^{\ast}\Sx)^{1/2}\right\|_{\mathcal{H}_K\rightarrow\mathcal{H}_K}\leq 1.
\end{align}

Recall the error decomposition
\begin{align*}
\left\|J_Kf^\lambda_{z,z^\nu}-f_{\rho}\right\|_{\rho}\nonumber
&\leq \left\|f_{\rho}-J_K(\lambda I+\Bnu^{\ast}\Bnu)^{-1}P_{z^\nu}J^*_Kf_{\rho}\right\|_{\rho}\nonumber  \\
&\quad +\left\|J_K(\lambda I+ \Bnu^{\ast}\Bnu)^{-1} P_{z^\nu}(J^*_Kf_{\rho}-\Sx^{\ast}\mathbb{Y})\right\|_{\rho}.
\end{align*}
The sample error, i.e. the second term on the right hand side of above inequality, can be estimated with the use of~\eqref{eq:J_K-bound}, \eqref{eq:3.1}, \eqref{eq:3.2} and~\eqref{eq:3.5} as follows
  \begin{align} \label{eq:3.7} &\left\| J_K(\lambda
      I+ \Bnu^{\ast}\Bnu)^{-1}P_{z^{\nu}}(J^*_Kf_{\rho}-\Sx^{\ast}\mathbb{Y})
    \right\|_{\rho}\nonumber
    \\
                               &\leq \left\| J_K(\lambda
                                 I+J^*_KJ_K)^{-1/2}\right\|_{\mathcal{H}_K\rightarrow
                                 L_2(X,\rho_X)}\nonumber
    \\
                               &\qquad \times\left\|(\lambda
                                 I+J^*_KJ_K)^{1/2}(\lambda I+\Sx^{\ast}\Sx
                                 )^{-1/2}\right\|_{\mathcal{H}_K\rightarrow\mathcal{H}_K}\nonumber
    \\
                               &\qquad \times \left\|(\lambda
                                 I+\Sx^{\ast}\Sx)^{1/2}(\lambda
                                 I+ \Bnu^{\ast}\Bnu)^{-1}P_{z^{\nu}}(\lambda
                                 I+\Sx^{\ast}\Sx)^{1/2}\right\|_{\mathcal{H}_K\rightarrow\mathcal{H}_K}\nonumber
    \\
                               &\qquad \times \left\|(\lambda
                                 I+\Sx^{\ast}\Sx)^{-1/2}(\lambda
                                 I+J^*_KJ_K)^{1/2}
                                 \right\|_{\mathcal{H}_K\rightarrow\mathcal{H}_K}\nonumber
    \\
                               &\qquad \times\left\|(\lambda
                                 I+J^*_KJ_K)^{-1/2}(J^*_Kf_{\rho}-\Sx^{\ast}\mathbb{Y})\right\|_{\mathcal{H}_K}\nonumber
    \\
                               &\leq 4c \log
                                 \dfrac{1}{\delta}\sqrt{\dfrac{\Ncl(\lambda)}{\left|z\right|}}
                                 \leq 4c \sqrt\lm \log
                                 \dfrac{1}{\delta}\left( 1 +
                                 \sqrt{\dfrac{\Ncl(\lambda)}{\lm\left|z\right|}}
                                 \right).
  \end{align}

 The rest of the proof is to estimate the
  approximation error, i.e.\ the first term
  on the right hand side of~\eqref{eq:3.6}. This can further be decomposed as
  \begin{align} \label{eq:3.8} \left\| f_{\rho}-J_K(\lambda
      I+\Bnu^{\ast}\Bnu)^{-1}P_{z^\nu}J^*_Kf_{\rho}\right\|\leq
    I_1+I_2,
  \end{align}	
  where
  \begin{align*}
    \begin{split}	
      I_1&=\left\|f_{\rho}-J_K(\lambda
        I+P_{z^\nu}J^*_KJ_KP_{z^\nu})^{-1}P_{z^\nu}J^*_Kf_{\rho}\right\|,
      \\
      I_2&=\left\| J_K[(\lambda I+P_{z^\nu}J^*_KJ_KP_{z^\nu})^{-1}-
        (\lambda I+\Bnu^{\ast}\Bnu)^{-1}]P_{z^\nu}J^*_Kf_{\rho}\right\|_{\rho}
      \\
      &=\left\| J_K(\lambda
        I+\Bnu^{\ast}\Bnu)^{-1}P_{z^\nu}(J^*_KJ_K- \Sx^{\ast}\Sx)
      \right.
      \\
      &\quad \times\left.P_{z^\nu}(\lambda
        I+P_{z^\nu}J^*_KJ_KP_{z^\nu})^{-1}P_{z^\nu}J^*_Kf_{\rho}\right\|_{\rho}.
    \end{split}
  \end{align*}
					
To estimate $I_1$ we recall~\eqref{eq:3.9}.			
Moreover, from \eqref{eq:2.2}, \eqref{eq:2.4} it follows that under
the condition \eqref{eq:2.5} we have
  \begin{align} \label{eq:3.10}
    &\left\|\varphi(J_KJ^*_K)-\varphi(J_KP_{z^\nu}J^*_K)\right\|_{L_2(X,\rho_X)\rightarrow
      L_2(X,\rho_X)}\nonumber
    \\
    &\quad \leq
      d_{\varphi}\varphi(\left\|J_K(I-P_{z^{\nu}})J^*_K\right\|_{L_2(X,\rho_X)\rightarrow
      L_2(X,\rho_X)})\leq c\varphi(\lambda).
  \end{align}

Then, using the source condition~\eqref{eq:2.1} with $f=f_{\rho}$,
and the qualification of KRR as in~(\ref{eq:quali}), we can estimate
$I_1$, by using~$\ra(t):= \lm/(\lm + t),\ t,\lm>0$, as follows
  \begin{align*}
    I_1&=\left\|(I-J_KP_{z^\nu}J^*_K(\lambda I+J_KP_{z^\nu}J^*_K)^{-1})f_{\rho}\right\|_{\rho}\nonumber
    \\
       &\leq\left\|\ra(J_KP_{z^\nu}J^*_K)\varphi(J_KP_{z^\nu}J^*_K)\upsilon_{f_{\rho}}\right\|_{\rho}\notag\\
& \qquad
         +\left\|\ra(J_KP_{z^\nu}J^*_K)(\varphi(J_KJ^*_K)-\varphi(J_KP_{z^\nu}J^*_K))\upsilon_{f_{\rho}}\right\|_{\rho}\nonumber
    \\
       &\leq \left\|\upsilon_{f_{\rho}}\right\|_{\rho}
			\kl{ \sup_{t>0}\ra(t)\varphi(t)
         +\left\|\varphi(J_KJ^*_K)-\varphi(J_KP_{z^\nu}J^*_K)\right\|_{L_2(X,\rho_X)\rightarrow L_2(X,\rho_X)}
				}
				\nonumber \\
       &         \leq c\varphi(\lambda).
  \end{align*}

  To estimate $I_2$ we observe that $I_2\leq I_{2,1}\cdot I_{2,2}$,
  where
  \begin{align*}
    I_{2,1}&=\left\|J_K(\lambda I+\Bnu^{\ast}\Bnu)^{-1}P_{z^{\nu}}(J^*_KJ_K-\Sx^{\ast}\Sx) \right\|_{\mathcal{H}_K\rightarrow\mathcal{H}_K},
    \\ I_{2,2}&=\left\|(\lambda I+P_{z^{\nu}}J^*_KJ_KP_{z^{\nu}})^{-1}P_{z^{\nu}}J^*_Kf_{\rho} \right\|_{\mathcal{H}_K}.
  \end{align*}
  By the same chain of arguments as in \eqref{eq:3.7} we obtain that
  \begin{align*}
    I_{2,1}\leq c\log \dfrac{1}{\delta}\sqrt{\dfrac{\Ncl(\lambda)}{\left|z\right|}},
  \end{align*}
  where the only difference is that one needs to use \eqref{eq:3.4}
  instead of \eqref{eq:3.5}.

Observing that $I_{2,2}\leq I_{2,2,1}+I_{2,2,2}$, we then have
  \begin{align*}
    I_{2,2,1}&=\left\|P_{z^{\nu}}J^*_K(\lambda I+J_KP_{z^{\nu}}J^*_K)^{-1}\varphi(J_KP_{z^{\nu}}J^*_K)\upsilon_{f_{\rho}}  \right\|_{\mathcal{H}_K},
    \\
    I_{2,2,2}&=\left\|P_{z^{\nu}}J^*_K(\lambda I+J_KP_{z^{\nu}}J^*_K)^{-1}(\varphi(J_KJ^*_K)-\varphi(J_KP_{z^{\nu}}J^*_K))\upsilon_{f_{\rho}}  \right\|_{\mathcal{H}_K}.
  \end{align*}
  Using~(\ref{eq:ga-bound}) we derive
  \begin{align*}
    I_{2,2,1}& =
               \|\ga(J_KP_{z^{\nu}}J^*_K)(J_KP_{z^{\nu}}J^*_K)^{1/2}\varphi(J_KP_{z^{\nu}}J^*_K)\|\left\|\upsilon_{f_{\rho}}\right\|_{\rho}\\
&\leq \left\|\upsilon_{f_{\rho}}\right\|_{\rho}
\sup_{t>0}\left|\ga(t)\right|t^{1/2}\varphi(t) \leq c \varphi(\lambda)\lambda^{-1/2}.
  \end{align*}

  Moreover, similarly to~(\ref{eq:J_K-bound}),  and \eqref{eq:3.10} there holds
  \begin{align}\label{eq:approxi}
    I_{2,2,2}\leq c \left\|\upsilon_{f_{\rho}}\right\|_{\rho} \varphi(\lambda)\sup_{t>0}\left\|\ga(t)\right|t^{1/2}\leq c\varphi(\lambda)\lambda^{-1/2}.
  \end{align}
  Thus, we have that~$I_{2,2}\leq c \varphi(\lambda)\lambda^{-1/2}$,
  and hence overall
\begin{align}\label{eq:approxi2}
I_2\leq c  \log\dfrac{1}{\delta}\varphi(\lambda)
    \sqrt{  \dfrac{\Ncl(\lambda)}{  \lm  \left|z\right|   }   },\quad \lm>0.
\end{align}
  Combining this with \eqref{eq:3.6}, \eqref{eq:3.7}, \eqref{eq:3.8}
  and \eqref{eq:3.1} we obtain the statement of the theorem, and the
  proof is complete.

\begin{remark}\label{rem:approxi}
We shall emphasize that in Theorem \ref{th:1} the total error estimate is dominated by the approximation error which is actually induced by the estimate of the term $I_{2,2,2}$ in (\ref{eq:approxi}). The misspecified source condition Assumption \ref{ass:smoothness} then yields an increasing function $\varphi(\lambda)\lambda^{-\frac{1}{2}}$ which blows up as $\lambda\rightarrow 0$. As a consequence, the estimate of $I_2$ in (\ref{eq:approxi2}) dominates the sample error.
\end{remark}

\subsection{Proof of Corollary~\ref{cor:lambda0}}
\label{sec:proof-prop}


  The right inequality in~(\ref{eq:lambda-minr}) is obvious by the
  choice of~$\lm_{0}$ from~(\ref{eq:lambda0}). To prove the left
  inequality we distinguish two cases. First, if $\lambda>\lambda_0$
  then $ \theta_{\varphi}(\lambda)>\varphi(\lambda)>\varphi(\lambda_0)$.
  Otherwise, if~$\lambda\leq\lambda_0$, then we use that by assumption the function
  $\lambda\rightarrow\varphi(\lambda)/\sqrt{\lambda}$ is decreasing,
  and hence
  \begin{align*}
    \theta_{\varphi}(\lambda)=
    \dfrac{\varphi(\lambda)}{\sqrt{\lambda}}\left(\sqrt{\lambda}+\sqrt{\dfrac{\Ncl(\lambda)}{\left|z\right|}}\right)\geq
    \dfrac{\varphi(\lambda)}{\sqrt{\lambda}}\sqrt{\dfrac{\Ncl(\lambda)}{\left|z\right|}}\geq
    \dfrac{\varphi(\lambda_0)}{\sqrt{\lambda_0}}\sqrt{\dfrac{\Ncl(\lambda_0)}{\left|z\right|}}=\varphi(\lambda_0).
  \end{align*}
This proves the left hand side bound and completes the proof of the
first assertion. The second one is an immediate application of the
theorem, and the proof is complete.




\end{document}